\documentclass{svproc}
\usepackage{fancyhdr}
\usepackage{makeidx}  
\usepackage{amsmath,amssymb}
\usepackage{algorithmic}
\usepackage{algorithm}
 \usepackage[pdftex]{graphicx}
 \usepackage{epstopdf}
\makeindex
\usepackage{url}

\begin{document}

\frontmatter          
\pagestyle{headings}  
\mainmatter              
\title{Fast Estimation of Relative Transformation Based on Fusion of Odometry and UWB Ranging Data}
\titlerunning{Fast Estimation of Relative Transformation}  
%
\author{Yuan Fu\inst{1,2} \and Zheng Zhang\inst{3}\and Guangyang Zeng\inst{4} \and Chun Liu\inst{1} \and Junfeng Wu\inst{4} \and Xiaoqiang Ren\inst{1,5}}
\authorrunning{Yuan Fu et al.} 
\institute{School of Mechatronic Engineering and Automation, Shanghai University, China, \email{xqren@shu.edu.cn}
\and
Shanghai Artificial Intelligence Laboratory, China
\and 
Institute of Artificial Intelligence, Shanghai University, China
\and
School of Data Science, the Chinese University of HongKong, Shenzhen, China
\and
Key Laboratory of Marine Intelligent Unmanned Swarm Technology and System, Ministry of Education}
\maketitle              
\begin{abstract}
In this paper, we investigate the problem of estimating the 4-DOF (three-dimensional position and orientation) robot-robot relative frame transformation using odometers and distance measurements between robots. Firstly, we apply a two-step estimation method based on maximum likelihood estimation. Specifically, a good initial value is obtained through unconstrained least squares and projection, followed by a more accurate estimate achieved through one-step Gauss-Newton iteration. Additionally, the optimal installation positions of Ultra-Wideband (UWB) are provided, and the minimum operating time under different quantities of UWB devices is determined. Simulation demonstrates that the two-step approach offers faster computation with guaranteed accuracy  while effectively addressing the relative transformation estimation problem within limited space constraints. Furthermore, this method can be applied to real-time relative transformation estimation when a specific number of UWB devices are installed.
\end{abstract}
%
\section{INTRODUCTION}\label{section_introduction}

In recent years, there has been a growing interest in multi-robot applications such as collaborative mapping of the environment~\cite{qin_autonomous_2019,potena_agricolmap_2019}, target tracking~\cite{yu_cooperative_2015}, and search and rescue~\cite{queralta_collaborative_2020}. To ensure the success of these operations, it is crucial to establish a common frame of reference for representing the attitude and sensor measurements of individual robots. Therefore, real-time estimation of relative orientation or initial relative frame transformation among robots, referred to as relative transformation estimation (RTE), becomes indispensable.

When prior knowledge of the environment is available, such as a pre-existing map or layout obtained from GPS and compass data, or fixed Ultra-Wideband (UWB) anchors, the robot's own pose can be easily determined along with the calculation of relative pose. However, in scenarios where environmental information is unknown, such as emergency response operations in compromised infrastructures, it becomes crucial to comprehend the relative positions of multiple robots for efficient exploration tasks. Consequently, an increasing number of researchers are employing radio signals, optical sensors, inertial measurements, and other techniques to achieve RTE without prior environmental information~\cite{khalajmehrabadi_modern_2017,he_wi-fi_2016,yassin_recent_2017}.

Currently, the prevailing approach involves equipping robots with external sensors such as cameras, LiDARs, radars and UWB sensors to obtain relative measurements of the robot's position, distance and attitude~\cite{zafari_survey_2019}. These approaches assume that adjacent robots already know the shape or size of each other. It is important to consider the varying complexity and cost of different sensors~\cite{de_ponte_muller_survey_2017}. When using cameras or LiDARs, performance is limited by factors like field of view and detection range of the sensor, environmental conditions including light intensity, rain or fog presence, complexity of the framework including data processing for object detection and tracking over time, and robustness against false detection~\cite{yuan_survey_2021}. Conversely, UWB sensors have advantages such as omni-directional measurement capability with centimeter-level accuracy at long distances (up to hundreds of meters), immunity to lighting conditions variations, ease in modeling and processing ranging data, and unique identification assignment for each UWB tag without ambiguity. Moreover, UWB provides a smaller and lighter sensor alternative for mobile robotics compared to LiDAR~\cite{xianjia_applications_2021}. However, for RTE tasks, a single ranging measurement is insufficient due to the lack of information about the surrounding environment, such as position. The majority of research papers therefore combine distance measurement with local self-measurement, and this paper primarily focuses on utilizing odometry for providing local location information.

The RTE problem based on odometry and UWB measurements involves optimizing a specific objective function, often focusing on the non-convex properties of the distance measurement model. Typically, local solvers such as the Gauss-Newton algorithm are employed for this problem~\cite{Kassakian:EECS-2006-64}. However, a significant drawback of these local solvers is their reliance on an appropriate initialization point~\cite{yan_review_2013,trawny_global_2010}. In the absence of an appropriate initialization point, the local solver may converge to a suboptimal solution. Convex relaxations using semidefinite programming (SDP) were introduced by~\cite{trawny_global_2010,goudar_optimal_2023,nguyen_relative_2023} to address this issue in the initial non-convex problem. These proposed methodologies specifically target a specific dimension of the state vector, which exhibits an inverse relationship with both computational efficiency and solution accuracy. An algebraic approach proposed in~\cite{trawny_interrobot_2010} aims to enhance calculation speed by substituting higher-order unknowns with first-order ones through introducing new variables. The resulting transformed equations are then solved using the same linear method as described in~\cite{molina_martel_unique_2019}. However, it is important to note that the algebraic method suffers from extensive substitution and approximate decomposition. In addition, both the algebraic approach and the linear method fail to consider measurement noise adequately, leading to potential compromises in accuracy. A nonlinear approach was introduced in~\cite{jiang_3-d_2020}, which prioritizes accuracy by identifying solutions with zero gradient and selecting the one with minimal variance, albeit at the expense of computational speed.

There are scenarios in which robots encounter limited mobility due to environmental or mission constraint~\cite{guo_ultra-wideband_2017}. In addition, some specific tasks may require quick solutions and allow only small movements. However, the local solvers, SDP and nonlinear approach methods require a certain amount of computation time, while the algebraic approach and linear method lack accuracy as mentioned above. Therefore, achieving high-precision estimates with high speed in limited space poses a significant challenge that requires careful attention. In this paper, our objective is to find an algorithm that minimizes computing time and reduces spatial dependence while effectively restraining estimation errors. We offer a method for determining the relative orientation of a mobile robot to its neighbors in a 2D plane and its position relative to its neighbors in a 3D plane by measuring relative distances and communicating displacement information. And we design the trajectory of robots and placement of UWB sensors, enabling flexible selection of varying quantities of UWB sensors to meet diverse application requirements, such as appropriately increasing their number when operating within limited motion spaces. We focus on addressing these challenges within four degrees of freedom  due to practical considerations, as estimating only the yaw angle is sufficient for most ground and aerial robots that possess built-in roll and pitch stabilization capabilities~\cite{goudar_optimal_2023}. Therefore, it is reasonable to focus solely on four degrees of freedom when considering the RTE problem. In addition, a shorter state vector facilitates easier handling of the RTE problem.

The main contributions of this paper are summarized as follows:
\begin{enumerate}
		\item  [$(i).$] We offer a two-step estimation method with a closed-form solution based on maximum likelihood estimation, which combines UWB distance measurements and odometer. Given the high sampling rates observed in real UWB systems, our work holds significant potential for various applications.
		\item  [$(ii).$] The path design rules and configuration positions of the UWB are thoroughly analyzed under varying numbers of UWBs, thereby ensuring a minimum operating time for robots.
\end{enumerate}

The remaining sections of this paper are organized as follows. Section~\ref{sec_system} presents the preliminaries and problem formulation, followed by the two-step estimation method in Section~\ref{sec_method}. Subsequently, the configurations of UWBs and the path design of the robots are discussed in Section~\ref{sec_path}. Simulation results in Section~\ref{sec_simulation} are then presented to verify the performance of our method compared to SDP and analyze the advantages and disadvantages of the proposed approaches. Finally, conclusions and future work are provided in Section~\ref{sec_conclusion}.




\emph{Notions}: All vectors are denoted by bold, lower case letters; matrices are represented by bold, upper case letters. The column vector obtained by stacking the columns of matrix $\bf A$ is denoted as $\bf{vec}(\bf A)$. The Kronecker product is symbolized by $\otimes$. The ceiling function, denoted as $\lceil x\rceil$, assigns the smallest integer that is greater than or equal to a given real number $x$. It is denoted as $X_n = o_p(a_n)$ that the sequence of $X_n/a_n$ converges to zero in probability. Rotations are expressed using elements from the special orthogonal group $SO(3) = \{{\bf R} \in {\mathbb R} ^{3\times 3}\mid{\bf R}^{\top}{\bf R}  =  {\bf I} _3, {\bf{det}}( {\bf R} ) = 1\}$, where ${\bf I}_d$ represents the identity matrix of dimension $d$, and ${\bf{det}}(\cdot)$ denotes the determinant operation. For a position vector ${\bf p} \in \mathbb R^3$, denote its elements as ${\bf p}=[P_{x},P_{y},P_{z}]^{\top}$.




\section{System Overview}\label{sec_system}
\subsection{Problem Statement}
    As illustrated in Fig.~\ref{fig_system}, consider host robot $\mathcal{R}_1$ and target robot $\mathcal{R}_2$, both moving in $3D$ and obtaining range measurements between each other. The host robot is equipped with $J_1$ UWB anchors, while the target robot carries $J_2$ UWB tags (the numbers of UWBs, $J_1, J_2$,  will be discussed in section~\ref{sec_path}). The time stamp of the odometry is denoted as $t_k$, where $k\in \{1,2,\cdots,K\}$. And $t_K$, where $K=M_2M_1$, denotes the minimum operating time, with $M_1=\lceil 4/{J_1}\rceil$ and $M_2=\lceil 3/{J_2}\rceil$ being design parameters discussed in Section~\ref{sec_path}.    
    \begin{figure}[thpb]
      \centering
      \includegraphics[scale=0.25]{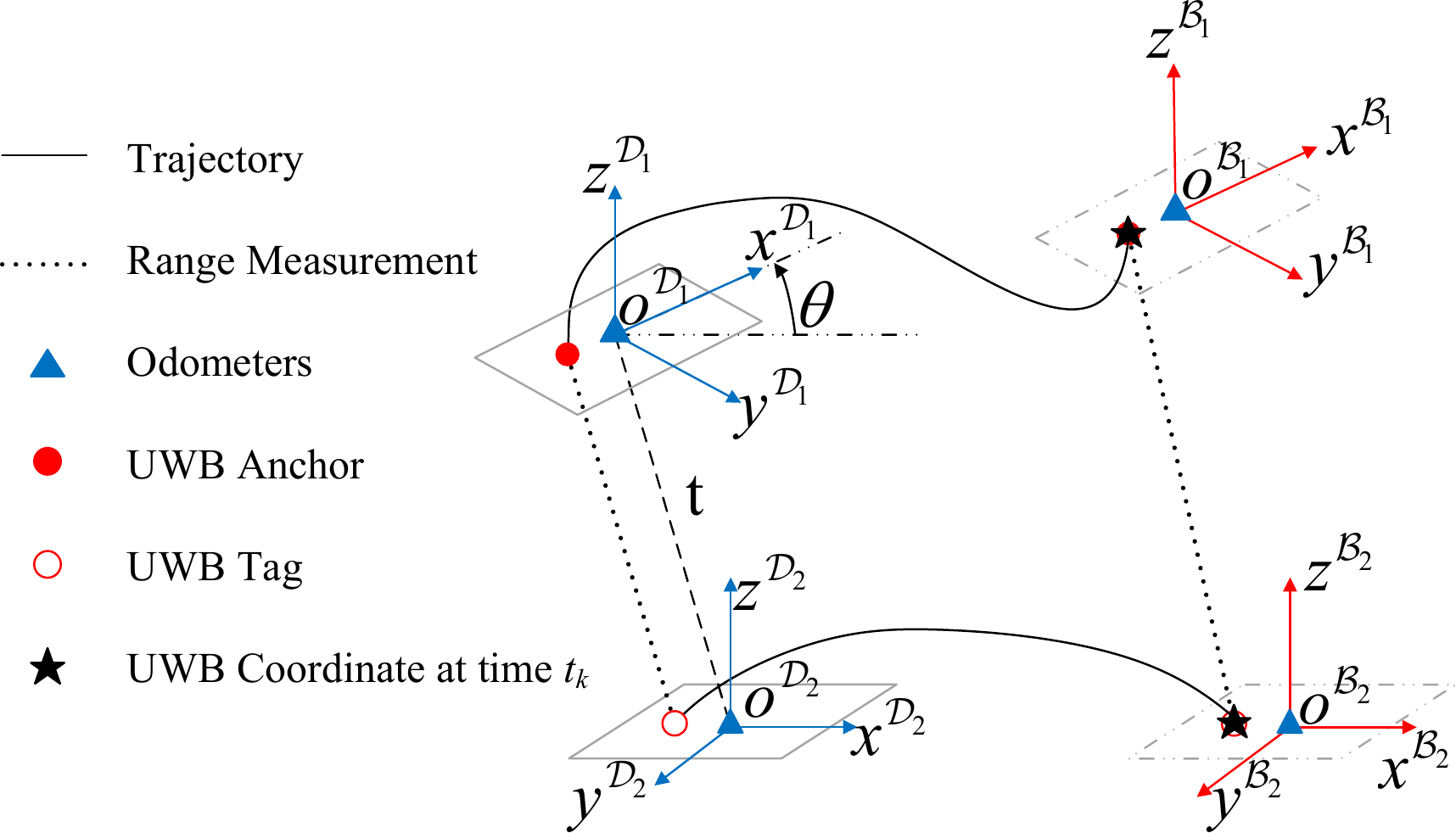}
      \caption{Overview of the proposed system.}
      \label{fig_system}
   \end{figure}
   
    Without loss of generality, we assume that the local odometer frames $\mathcal{D}_1$ and $\mathcal{D}_2$ at $t_1=0$ serve as the  global frames of each robot. 
    It is assumed that the $z$ axis of frame ${\mathcal D}_n, n \in \{1,2\}$ is always parallel to the direction of gravity. And the UWB localization of robot ${\mathcal R}_n$ is established within its local body frame, denoted as ${\mathcal B}_n$. Further, it is assumed that each robot has the capability to determine its own position relative to a global frame of reference by utilizing odometry. Therefore, we assume that the coordinates of the $j_n$-th UWB on the robot, ${\bf p}_{\mathcal{B}_n}^{j_n}$, is known. Accordingly, at time $t_k$, the $j_n$-th UWB position in the local odometer frame $\mathcal{D}_n$ is 
\begin{equation}
    {\bf p}_{n,\mathcal{D}_n}^{j_n,k}=_{{\mathcal B}_n}^{{\mathcal D}_n}{\bf p}_k+_{{\mathcal B}_n}^{{\mathcal D}_n}{\bf R}_k{\bf p}_{\mathcal{B}_n}^{j_n},
\end{equation}
where $j_n\in \{1,2,\cdots,J_n\}$, $_{{\mathcal B}_n}^{{\mathcal D}_n}{\bf R}$ and $_{{\mathcal B}_n}^{{\mathcal D}_n}{\bf p}$ are rotation matrix and position vector from frame ${\mathcal B}_n$ to ${\mathcal D}_n$ obtained by odometry, respectively.  

The objective is to find the 4 degree-of-freedom transformation between frames $\mathcal{D}_1$ and $\mathcal{D}_2$, i.e., the translation ${\bf t} \in {\mathbb R} ^{3}$ and rotation ${\bf R}\in SO(3)$, by using UWB and odometer data, without considering the issue of odometer drift in this paper. In other words, we assume that the provided odometer data is precise. Let $\theta$ denote the relative yaw angle between ${\mathcal D}_1$ and ${\mathcal D}_2$ [see Fig.~\ref{fig_system}]. Then $\bf R$ can be calculated as the basic 3-D rotation matrix around the $z$-axis by an angle $\theta\in[0,2\pi)$ as
\begin{equation}
	{{\bf R}}({\theta})\triangleq \begin{bmatrix}
        \cos{{\theta}} & -\sin{{\theta}}~&0\\
	\sin{{\theta}} & \cos{{\theta}}~&0\\
        0&~0~&1	\end{bmatrix}= \begin{bmatrix}
        \tilde {\bf R}({\theta})&0\\
	0& 1	\end{bmatrix}.
 \end{equation} 
 with $\tilde {\bf R}=\begin{bmatrix}
        \cos{{\theta}} & -\sin{{\theta}}\\
	\sin{{\theta}} & \cos{{\theta}}	\end{bmatrix}$ for simplicity.
   
\subsection{Range Measurement Model}
The anchor positions and tag positions are obtained from the respective robot odometers, and the range measurements $d_{i}^{j_1 j_2,k}$ are obtained from UWB configured on the robots. We have range measurement model as follows:
\begin{equation}\label{eqn_range_difference_measurements_model}
	\begin{gathered}
	d_{i}^{j_1 j_2,k}=\left\|{\bf p}_{1,\mathcal{D}_1}^{j_1,k}-{\bf p}_{2,\mathcal{D}_1}^{j_2,k}\right\|+r_{i}^{j_1 j_2,k}=\left\|{\bf p}_{1,\mathcal{D}_1}^{j_1,k}-{{\bf R}} {\bf p}_{2,\mathcal{D}_2}^{j_2,k}-{\bf t}\right\|+r_{i}^{j_1 j_2,k} ,
	\end{gathered}
\end{equation}
where $i\in \{1,2,\cdots,N\}$, $r_{i}^{j_1 j_2,k}$ are i.i.d. Gaussian white noise with zero mean and finite and known variance $\sigma_{j_1j_2}^2>0$ and ${\bf p}_{2,\mathcal{D}_1}^{j_2,k}$ is the UWB position in the robot $\mathcal{R}_1$'s odometer frame $\mathcal{D}_1$. The UWB anchors receive signals from all UWB tags and relying on the high transmission rates of UWB technology~\cite{9255175}, each pair of anchor and tag can generate $N$ distance measurements corresponding to odometer values at time $t_k$.

The RTE problem is formulated as follows: given ${\bf p}_{n,\mathcal{D}_n}^{j_n,k}$ and the distance measurements $d_{i}^{j_1 j_2,k}$, estimate the pose of $\mathcal{R}_2$ in the odometer frame $\mathcal{D}_1$, i. e., the translation ${\bf t} $ and rotation ${\bf R}$. 

\section{Proposed Approach}\label{sec_method}
For the ease of notation, let $d_{i}^{j_1 j_2,k}$ be denoted as $d_{i}^{l_1,l_2}$, ${\bf p}_n^{j_n,k}$ as ${\bf p}_n^{l_n}$, $r_{i}^{j_1 j_2,k}$ as $r_{i}^{l_1, l_2}$, ${\bf p}_{1,\mathcal{D}_1}^{j_1,k}$ as ${\bf p}_1^{j_1,k}$ and ${\bf p}_{2,\mathcal{D}_2}^{j_2,k}$ as ${\bf p}_2^{j_2,k}$. Herein, we define $l_n=\{x|x=J_n(k-1)+j_n,~k=1,2,\cdots,M_n\}$. The distance measurement $d_{i}^{l_1,l_2}$ exists for each pair of $l_1, l_2$ based on our meticulously designed path of robots in Section~\ref{sec_path}.
\subsection{Maximum Likelihood Formulation} 
With $N_1=M_1J_1$ and $N_2=M_2J_2$ for simplicity, the joint estimation of orientation and translation, according to the principle of maximum likelihood estimation (ML), can be regarded as solving the optimization problem as follows:
\begin{equation}\label{eqn_MJ_Problems}
	\begin{aligned}
        \min_{{\bf R},{\bf t}}~ &\sum_{l_2=1}^{N_2}\sum_{i=1}^{N} \sum_{l1=1}^{N_1} \cfrac{\left(d_{i}^{l_1,l_2}-\left\|{\bf p}_1^{l_1}-{{\bf R}} {\bf p}_2^{l_2}-{\bf t}\right\|\right)^{2}}{\sigma_{l_1, l_2}^{2}}\\
        \text{s.t.}~ &{\bf R}\in SO(3), {\bf t}\in \mathbb{R}^3.
	\end{aligned}
\end{equation}
Then we can obtain the solution $\hat{\bf R}_{ML}$, $\hat{\bf t}_{ML}$ of this optimization problem. Let $n=NN_1N_2$, we have $n$ measurements, in total. 

The non-linearity and non-convexity of the optimization problem~\eqref{eqn_MJ_Problems} make it difficult to solve. Therefore, we will linearize the model next, which can then be solved using a linear least-squares estimator. The solution obtained from this estimator can subsequently be used as the initial value in the second step.
\subsection{First step} 
The non-linearity of the model in~\eqref{eqn_range_difference_measurements_model} necessitates linearization, which is achieved by squaring it as follows: 
\begin{equation}\label{eqn_Squared_model}
	\begin{aligned}
		{d_{i}^{l_1, l_2}}^2=&(\left\|{\bf p}_1^{l_1}-{{\bf R}} {\bf p}_2^{l_2}-{\bf t}\right\|+r_{i}^{l_1, l_2})^2 \\
        =&\left\|{\bf p}_1^{l_1}\right\|^{2}-2 {{\bf p}_1^{l_1}}^{\top}\left({{\bf R}} {\bf p}_2^{l_2}+{\bf t}\right)+\left\|{{\bf R}} {\bf p}_2^{l_2}+{\bf t}\right\|^{2}\\
	&+{r_{i}^{l_1, l_2}}^{2}+2\left\|{\bf p}_1^{l_1}-{{\bf R}} {\bf p}_2^{l_2}-{\bf t}\right\| {r_{i}^{l_1, l_2}},\\
	\end{aligned}
\end{equation}
where the noise term ${r_{i}^{l_1, l_2}}^{2}+2\left\|{\bf p}_1^{l_1}-{{\bf R}} {\bf p}_2^{l_2}-{\bf t}\right\| r_{i}^{l_1, l_2}$ has a non-negative known mean $\sigma_{l_1, l_2}^2$. Subtracting $\sigma_{l_1, l_2}^2$ from both sides of equation~\eqref{eqn_Squared_model}, we have
\begin{equation}\label{eqn_Squared_model_noise}
	\begin{aligned}
        {d_{i}^{l_1, l_2}}^2-\sigma_{l_1, l_2}^2=&\left\|{\bf p}_1^{l_1}\right\|^{2}-2 {{\bf p}_1^{l_1}}^{\top}\left({{\bf R}} {\bf p}_2^{l_2}+{\bf t}\right)\\
	&+\left\|{{\bf R}} {\bf p}_2^{l_2}+{\bf t}\right\|^{2}+e_{i}^{l_1, l_2},\\
	\end{aligned}
\end{equation}
where $e_{i}^{l_1, l_2}=2\left\|{\bf p}_1^{l_1}-{{\bf R}} {\bf p}_2  ^{l_2}-{\bf t}\right\| r_{i}^{l_1, l_2}+({r_{i}^{l_1, l_2}}^2-\sigma_{l_1, l_2}^2)$ is the new zero-mean noise term. Stacking \eqref{eqn_Squared_model_noise} for all the $M_1J_1$ anchors and subtracting the knowns in~\eqref{eqn_Squared_model_noise} from the measurements, we obtain:
\begin{equation}\label{eqn_with_NL1}
	\begin{gathered}
        {\bf d}  ^{l_2}=\left\|{{\bf R}} {\bf p}_2^{l_2}+{\bf t}\right\|^{2} {\bf 1}_{NM_1 J_1} -2 {\bf p}_{1}^{\top}\left({{\bf R}} {\bf p}_2^{l_2}+{\bf t}\right)+{\bf e}^{l_2},\\
	\end{gathered}
\end{equation}
where
\begin{align*}
        {\bf d}  ^{l_2}&=\begin{bmatrix}
	  {d^{1, l_2}}^2-{\sigma_{1, l_2}}^2 \\
        \vdots \\
        {d^{{N_1}, l_2}}^2-{\sigma_{{N_1}, l_2}}^2 \\
        \end{bmatrix}-{\bf p}_{{\bf n}1}^{\top},~{\bf e}^{l_2}=\begin{bmatrix}
	  e^{1, l_2} \\
        \vdots \\
        e^{{N_1}, l_2} \\
        \end{bmatrix}\\
        {\bf p}_{{\bf n}1}&=\left[\left\|   {\bf p}_1 ^{1}\right\|^{2},\cdots,\left\|  {\bf p}_1^{N_1}\right\|^2\right]\otimes{1_N}^{\top}\in \mathbb{R}^{1 \times NN_1},\\
        {\bf p}_{1}&=\left[ {\bf p}_{1}^{1},\cdots,  {\bf p}_1^{N_1}\right]\otimes{1_N}^{\top}\in \mathbb{R}^{3\times N N_1}.
\end{align*}
To eliminate the second-order terms $\left\|{{\bf R}} {\bf p}_2^{l_2}+{\bf t}\right\|^{2} {\bf 1}_{NN_1}$ in the least-squares problem~\eqref{eqn_with_NL1}, we introduce the projection matrix ${\bf P}={\bf I}_{N N_1}-\cfrac{{\bf 1}_{N N_1}{\bf 1}_{N N_1}^{\top}}{N N_1}.$ Multiply ${\bf P}$ on each side of the equation~\eqref{eqn_with_NL1}, all terms with ${\bf 1}_{NN_1}$ equals to $\bf 0$ and the least square problem become:
\begin{equation}\label{eqn_SLS}
	\begin{aligned}
        \min_{{\bf R},{\bf t}}~ &\sum_{l_2=1}^{N_2} \left\|{ {\bf \bar d}  ^{l_2}+2\bar {\bf p}_1^{\top}\left({{\bf R}} {\bf p}_2^{l_2}+{\bf t}\right) }\right \|^2_{\Sigma _{{\bf \bar e}^{l_2}}}\\
        \text{s.t.} ~ &{\bf R}\in SO(3), {\bf t}\in \mathbb{R}^3,
	\end{aligned}
\end{equation}
where we define ${\bf P d}  ^{l_2}={\bf \bar d}  ^{l_2},{\bf p}_1{\bf P}=\bar {\bf p}_{1}$ and $ {\bf P}{\bf e}^{l_2}={\bf \bar e}^{l_2}.$ 

The expression $\left({{\bf R}} {\bf p}_2^{l_2}+{\bf t}\right)$ can be rewritten as
\begin{align*}
   {\bf \gamma}_1\left(\tilde {\bf R}{\bf \gamma}_1^{\top}{\bf p}_2^{l_2}+{\bf \gamma}_1^{\top}{\bf t}\right)+
{\bf \gamma}_2\left({\bf \gamma}_2^{\top}{\bf p}_2^{l_2}+
{\bf \gamma}_2^{\top}{\bf t}\right)
={\bf \gamma}_1\tilde {\bf R}{\bf \gamma}_1^{\top}{\bf p}_2^{l_2}+{\bf \gamma}_2{\bf \gamma}_2^{\top}{\bf p}_2^{l_2}+{\bf t}
\end{align*}
where we have introduced the matrices ${\bf \gamma}_1=\left[\begin{matrix}1~0~0\\
0~1~0\end{matrix}\right]^{\top}$ and ${\bf \gamma}_2=[0~0~1]^{\top}$. By utilizing the matrix property $\bf vec(ABC)=C^{\top}\otimes A{\bf vec}(B)$, stacking~\eqref{eqn_SLS} for all the $N_2$ tags and vectorizing yields the following optimization problem:
\begin{equation}\label{eqn_with_vec}
        \begin{aligned}
        & \min_{\tilde {\bf R},{\bf t}}~  \left\|{\bf d}   +2{\bf 1}_{N_2}\otimes\bar {\bf p}_1^{\top}{\bf t}+2{{\bf p}_2  }^{\top} \otimes\bar {\bf p}_1^{\top}{\bf{vec}}({\bf \gamma}_1 \tilde {\bf R}{\bf \gamma}_1^{\top})\right\|^2_{\Sigma _{{\bf \bar e}}} \\
      \text{s.t.} ~ \tilde{{\bf R}}\in SO(2), {\bf t}\in \mathbb{R}^3,
        \end{aligned}
\end{equation}
where ${\bf d}=\left[{\bf \bar d}  ^{1};\ldots;{\bf \bar d}  ^{N_2}\right]+2{\bf p}_2^{\top}\otimes\bar {\bf p}_1^{\top} {\bf {vec}}({\bf \gamma}_2{\bf \gamma}_2^{\top})\in \mathbb{R}^{n\times1}$, ${\bf \bar e}=\begin{bmatrix}
	  {\bf \bar e}^{1\top},\cdots, {\bf \bar e}^{N_2\top}
        \end{bmatrix}^{\top}$, and ${\bf p}_{2}=\left[
        {\bf p}_2^{1}~  \dots~   {\bf p}_2^{N_2}
        \right]\in \mathbb{R}^{3\times N_2}$.
        
We define $\bf x$ equals $\begin{bmatrix}
	x_1  ~x_2\end{bmatrix}^{\top}$, where $x_1$ equals $\sin(\theta)$ and $x_2$ equals $\cos(\theta)$. The unconstrained linear least squares in~\eqref{eqn_with_vec} can then be compactly expressed as 
\begin{equation}\label{eqn_UdLS2}	
\min_{\bf y} \left\|{\bf  d}-{\bf H}{\bf }\right\|^2_{ \Sigma_{\bf \bar e}},
\end{equation}
where ${\bf y} = [{\bf x}^{\top},{\bf t}^{\top}]^{\top}$,
         ${\bf H}=\left[{\bf H}_{1}~{\bf H}_{2}\right]$, $ 
        {\bf H}_{1}=-2 \left({\bf p}_2^{\top}\otimes \bar {\bf p}_{1}^{\top}\right) \mathcal{T}$, $
        {\bf H}_{2}=-2\left({\bf 1}_{N_2} \otimes \bar {\bf p}_1^{\top}\right)$
        and $\mathcal{T}=\left[\begin{array}{ccccccccc}
			0 & 1 & 0 & -1 & 0 & 0 & 0 & 0 & 0 \\
			1 & 0 & 0 & 0  & 1 & 0 & 0 & 0 & 0 
		\end{array}\right]^{\top}$.

From \eqref{eqn_UdLS2}, we get unconstrained least squares estimate
\begin{equation}\label{eqn_Closed_form_solution}
	\begin{gathered}
		{\bf \hat y}=\left[\begin{array}{l}
			\bf \hat x \\
			\bf \hat t \\
		\end{array}\right]=\left({\bf H}^{\top} {\bf H}\right)^{-1} {\bf H}^{\top} {\bf d},
	\end{gathered}
\end{equation}
Since the estimate $\hat{\bf y}$ in \eqref{eqn_Closed_form_solution} lacks constraints, we project $\hat{\tilde{{\bf R}}}$ on to $SO(2)$ to obtain
\begin{equation}
    \pi(\hat{\tilde{{\bf R}}})={\bf{arg}} \min_{{\bf W}\in SO(2)}\|\hat{\tilde{{\bf R}}}-{\bf W}\|^2_F={\bf U}{\text{diag}}([1,\text{det}({\bf U}{\bf V}^{\top})]){\bf V}^{\top},
\end{equation}
where $\pi$ is the matrix projection that maps an arbitrary matrix onto $SO(2)$, $\bf U$ and $ V$ are obtained from the singular value decomposition of $\hat{\tilde{{\bf R}}}=\bf U\Sigma V^{\top}$. Consequently, we get the final solution $\pi(\hat{\tilde{{\bf R}}})={\bf R(\hat{\theta})}$ of the first step, which serves as the initial value for one step of the Gauss-Newton iteration.
\subsection{One Step of the Gauss-Newton Iteration}
We have known ${\bf R(\hat{\theta})}$ and $\bf t$ are $\sqrt{n}$-consistent~\cite{jiang_efficient_2023}, so we can implement one step of Gauss-Newton iteration to refine the estimate.  

The problem~\eqref{eqn_MJ_Problems} can be rewritten as an unconstrained one 
\begin{equation}\label{eqn_MJ_GN}
	\begin{aligned}
        \min_{{\bf \theta},{\bf t}}~ &\sum_{l_2=1}^{N_2}\sum_{i=1}^{N} \sum_{l1=1}^{N_1} \cfrac{\left(d_{i}^{l_1,l_2}-\left\|{\bf p}_1^{l_1}-{\bf L}_i{\bf{vec}}{{\bf R(\hat{\theta}+\theta)}} -{\bf t}\right\|\right)^{2}}{\sigma_{l_1, l_2}^{2}},
	\end{aligned}
\end{equation}
where ${\bf L}_i=({\bf p}_2^{l_2}\otimes {\bf I}_3)^{\top}$.
Jacobian matrix $\bf J$ can be calculated by
\begin{align*}
        {\bf J}^{\top}=\left[\begin{array}{ccc}
			\dfrac{\partial \|f_{1,1,1}^{(0)}\|}{\partial ({\theta},{\bf t})}, & \cdots, & \dfrac{\partial \|f_{N_1,N_2,N}^{(0)}\|}{\partial ({\theta},{\bf t})}\\
		\end{array}\right],
\end{align*}
where
\begin{align*}
		\cfrac{\partial \|f_{{{l}_1},{{l}_2},i}^{(0)}\|}{\partial ({\theta},{\bf t})}&=\begin{bmatrix}
			-\cfrac{1}{\|f_{{{l}_1},{{l}_2},i}^{(0)}\|}
			\Psi^{\top} ({\bf I}_2\otimes {\bf \hat R}^{\top})  {{\bf L}}_i^{\top}\,f_{{{l}_1},{{l}_2},i}^{(0)},~
			-\cfrac{1}{\|f_{{{l}_1},{{l}_2},i}^{(0)}\|} f_{{{l}_1},{{l}_2},i}^{(0)}
		\end{bmatrix}^{\top},\\
		\Psi&=\cfrac{\partial {\rm vec}({R}(0))}{
			\partial {\theta}}=\begin{bmatrix}
			0~1~0~-1~0~0~0~0~1
		\end{bmatrix}^{\top}.
\end{align*}
Thus the one step Gauss-Newton equation with $\Sigma_{n}$ as covariance matrix for $r_{{{l}_1}{{l}_2}}$ becomes
\begin{equation}\label{eqn_one_step_GN}
	(\hat{{\theta}},\hat{{\bf t}})_{\rm GN}=
	(0,\hat{\bf t})+
	({\bf J}^{\top}\Sigma_{n}^{-1}{\bf J})^{-1}{\bf J}^{\top}\Sigma_{n}^{-1}\big({\bf d}-{\bf f}(0,\hat{{\bf t}})\big).
\end{equation}
The estimates can subsequently be obtained as follows:
\begin{equation}\label{eqn_GN}
    \hat{\bf R}_{GN}={\bf R}(\hat{\theta}+\hat{\theta}_{GN}),~\hat{\bf t}_{GN}=\hat{\bf t}_{GN}.
\end{equation}
\begin{theorem}
The estimate obtained through one step of Gauss-Newton iteration on the ML problem~\eqref{eqn_MJ_Problems} converges in probability to the ML estimates as the number of measurements $n$ increases. Specifically, we have 
$$\hat{\bf R}_{GN}-\hat{\bf R}_{ML}=o_p(1/\sqrt{n}),~\hat{\bf t}_{GN}-\hat{\bf t}_{ML}=o_p(1/\sqrt{n}).$$
\label{the_gny}
\end{theorem}
While our paper considers the four degrees of freedom RTE problem as distinct from that discussed in paper~\cite{jiang_efficient_2023}, Theorem~\ref{the_gny} can be proven using a similar approach to Theorem 1 presented in paper~\cite{jiang_efficient_2023}. Thus, the proof is omitted here. 
We summarize the two-step method in Algorithm~\ref{Algorithm_1}.

\begin{algorithm}
	\caption{Two-step method}
	\label{Algorithm_1}
        \begin{algorithmic}[1]
		\renewcommand{\algorithmicrequire}{ \textbf{Input:}}
		\renewcommand{\algorithmicensure}{ \textbf{Output:}} 
		\REQUIRE $d_{i}^{l_1,l_2}$, ${\bf p}_1^{l_1}$, ${\bf p}_2^{l_2}$ and ${\sigma_{l_1, l_2}^{2}}$.
		\ENSURE the estimates of ${\bf R}_o$ and $t_o$.
		\STATE Construct ${\bf d}$ and $\bf H$;
		\STATE Calculate $\hat {\bf y}$ according to~\eqref{eqn_Closed_form_solution};
		\STATE Project $\hat{\tilde {\bf R}}$ onto $SO(2)$;
		\STATE Construct $\bf J^{\top}$ and ${\bf d}-{\bf f}(0,\hat{{\bf t}})$;
		\STATE Calculate $(\hat{{\theta}},\hat{{\bf t}})_{\rm GN}$ according to~\eqref{eqn_one_step_GN};
		\STATE Calculate $\hat{\bf R}_{GN}$ and $\hat{\bf t}_{GN}$ based on the equation~\eqref{eqn_GN}.
	\end{algorithmic}
\end{algorithm}

\section{Path Design and UWB Configuration}\label{sec_path}
Due to the dependence of~\eqref{eqn_Closed_form_solution} on the full rank of matrix $\bf H$, the following theorem is presented. 
\begin{theorem}
The matrix $\bf H$, defined in~\eqref{eqn_Closed_form_solution}, has full column rank, ensuring the existence of the solution $\bf\hat{y}$, when both robots' paths meet the following conditions:

\begin{enumerate}
    \item If $J_1<4$: Robot $\mathcal{R}_1$ generates $N_1$ non-coplanar coordinates through its odometry by moving at time interval $t_{1}-t_{M_1}$. This path is then repeated $M_2$ times resulting in $M_2$ sets of coordinates.

    \item If $J_1=4$: Robot $\mathcal{R}_1$, remaining static in its initial position, install four anchors non-coplanarly.

    \item If $J_2<3$: Robot $\mathcal{R}_2$, with each movement occurring at time interval $t_{iM_1-1}-t_{iM_1},~i\in \{2,\cdots,M_2-1\}$, moves $M_2$ times generating $N_2$ non-coplanar coordinates with the origin $O^{\mathcal{D}_2}$.

    \item If $J_2=3$: Robot $\mathcal{R}_2$, remaining static in its initial position, installs three tags non-coplanarly with the origin $O^{\mathcal{B}_2}$.
\end{enumerate}
where $M_1=\lceil 4/J_1\rceil$ and $M_2=\lceil 3/J_2\rceil$. 
\label{the_path} 
\end{theorem}  

\begin{proof}
Four distinct anchor locations are necessary because four spheres can uniquely determine a position.
Rewrite matrix $\bf H$ as
$$
\bf H=-2 \left[\begin{array}{cc}
       {\bf p}_2^{\top}\otimes \bar {\bf p}_1^{\top}&1_{N_2}\otimes \bar {\bf p}_1^{\top}\\
\end{array}\right]\left[\begin{array}{cc}
       \mathcal{T} &0\\
       0           &{\bf I}_3
\end{array}\right].
$$
From the non-coplanarity of the aforementioned conditions, we have ${\bf{Rank}}({\bf p}_2)=3$ and ${\bf{Rank}}(\bar {\bf p}_1)={\bf{Rank}}\left({\bf p}_1-{\bf {avg}}({\bf p}_1){\bf 1}_{N_1}\right)=3$. Consequently, it follows that $\bf{Rank}({\bf p}_2^{\top}\otimes \bar {\bf p}_1^{\top})=9$, and the rank of matrix $\bf H$ can be expressed as 
$$\bf{Rank}(\bf H)=\bf{Rank}\left(\bf H_{1}[{\bf I}_9~ C]\left[\begin{array}{cc}
       \mathcal{T} &0\\
       0           &{\bf I}_3
\end{array}\right]\right).$$
Assume that $\bf A$ is an $m\times n$ matrix, if $\bf B$ is an $l\times m$ matrix of rank $m$, then $\bf{Rank}(BA)={\bf{Rank}}(A)$. It is known that the rank of matrix $\mathbf{H}_1$ is 9, where ${\bf H}_1\in \mathbb{R}^{n\times 9}$.To establish the full column rank of $[\mathbf{I}_9~ \mathbf{C}] \begin {bmatrix}\mathcal{T}&0\\0&\mathbf {I}_3\end {bmatrix}$ would suffice to prove that matrix $\mathbf {H}$ is also full rank.

Since ${\bf p}_2^{\top}$ is full rank and ${\bf 1}_{N_2}$ can be represented as ${\bf p}_2^{\top}[\alpha,\beta,\gamma]^{\top}$, where $[\alpha,\beta,\gamma]^{\top}$ is a non-zero vector, we can deduce that $\bf H_2$ can be expressed as follows:
\begin{equation*}
	\begin{aligned}
 \bf H_2=-2({\bf p}_2^{\top}\left[\begin{array}{c}
       \alpha\\
       \beta\\
       {\bf \gamma}
\end{array}\right])\otimes(\bar {\bf p}_1^{\top}{\bf I}_3)
=-2({\bf p}_2{\top}\otimes \bar {\bf p}_1^{\top})(\left[\begin{array}{c}
       \alpha\\
       \beta\\
       {\bf \gamma}
\end{array}\right]\otimes {\bf I}_2).
	\end{aligned}
\end{equation*}
As a result, we have ${\bf C}=\left[\begin{matrix}
       \alpha~       \beta      ~ {\bf \gamma}
\end{matrix}\right]^{\top}\otimes {\bf I}_2$, and it can be observed that $$
[{\bf I}_9~ C]\left[\begin{array}{cc}
       \mathcal{T} &0\\
       0           &{\bf I}_3
\end{array}\right]=\left[\begin{matrix}
0&1&0&-1&0&0&0&0&0\\
1&0&0&0&1&0&0&0&0\\
\alpha&0&0&\beta&0&0&{\bf \gamma}&0&0\\
0&\alpha&0&0&\beta&0&0&\gamma&0\\
0&0&\alpha&0&0&\beta&0&0&\gamma
\end{matrix}\right]^{\top},$$
is full rank. 
\end{proof}
\begin{remark}
    After setting $J_1=4$ and $J_2=3$, it can be observed that real-time RTE can be achieved even in the absence of an odometer. 
\end{remark}

\section{Simulation}\label{sec_simulation}
In this section, we will demonstrate the efficacy of our approach through simulations. Our simulations are specifically designed to isolate and evaluate various factors that can impact system performance: signal noise, trajectory configuration, system movement speed and computation time. 

Considering that~\cite{nguyen_relative_2023} has already conducted a comprehensive comparison with other state-of-the-art methods, and SDP has demonstrated superior accuracy in challenging environments, we exclusively focus on comparing our proposed method (two-step method) with the SDP method in~\cite{nguyen_relative_2023} by evaluating the root mean square error (RMSE):
$$\text{RMSE}({\bf t}) =\sqrt{\cfrac{1}{L}\sum^L_{l=1}\|\hat{\bf t}-{\bf t}_o\|^2_2},~\text{RMSE}({\bf R}) =\sqrt{\cfrac{1}{L}\sum^L_{l=1}\|\hat{\bf R}-{\bf R}_o\|^2_F},$$
where $L=1000$ is the number of of Monte Carlo experiments. All experiments are run on a laptop with an Intel Core i5-11500 CPU with 16 GB RAM.
\subsection{Simulation Setup}
We provide two cases as examples for simulation:
\begin{enumerate}
\item  [$(i).$] There is $J_1=1$ anchor and $J_2=1$ tag deployed at $[0,0,10]$ in their respective  body frame.
\item  [$(ii).$] There are $J_1=4$ anchors deployed at $[0,0,0]$, $[10,0,0]$, $[0,10,0]$ and $[0,0,10]$ in the body frame. There are $J_2=3$ tags deployed at $[10,0,0]$, $[0,10,0]$ and $[0,0,10]$ in the body frame.
\end{enumerate}
Unless explicitly specified otherwise, the relevant parameters will be configured as follows. The noise variance $\sigma_{l_1,l_2}$ is set to the same value $\sigma=1$ for different $l_1,l_2$ combinations. The true values are $t_o=[20,20,20]$ and $\theta_o=60^\circ$. The number of repeated measurements is fixed at $N=100$ for each time $t_k$. It should be noted that in~\cite{nguyen_relative_2023}, the SDP was performed without employing repeated measurements. However, to ensure fairness, we also perform the SDP with $N=100$ repetitions. The maximum distance to the origin $O^{\mathcal{D}_n}$, denoted as the maximum operating distance $R_{max}$, is set as 10. In each simulation, as per the conditions specified in Theorem~\ref{the_path}, the robots' trajectories are randomly generated, ensuring that each trajectory has a distance to the origin not exceeding $R_{max}$. 

\subsection{Simulation Results}
$1)$ \emph{Effect of noise levels}:
We consider $\sigma$ values of $[0.01, 0.1, 1, 10, 100]$. To assess the efficacy of our proposed approach and the SDP in estimating initial relative transformation, we perform simulations under various cases with different variances.  

The comparison results of the two methods in case $(i)$ and the two-step method in case $(ii)$ are presented in Fig.~\ref{fig_compare_1}. It can be observed that as the noise variance $\sigma$ increases, the rate of improvement in the two-step method is slower compared to SDP. This suggests that the two-step method is particularly well-suited for scenarios characterized by a high value of $\sigma$. And the RMSE value is smaller in case $(ii)$ of the two-step method in Fig.~\ref{fig_compare_1}. This is because case $(ii)$ does not have the error introduced by the movement, unlike situation $(i)$.   This suggests that, with the same amount of data,  minimizing movement is beneficial for improving estimation accuracy.
\begin{figure}[thpb]
      \centering
      \includegraphics[width=\textwidth, keepaspectratio]{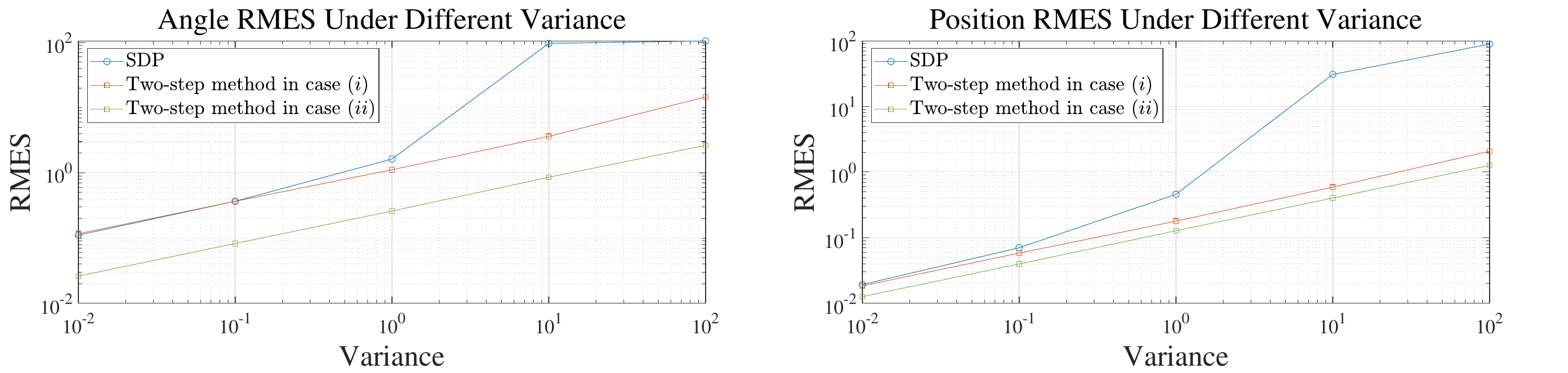}
      \caption{Performance under repeated ranging in case $(i)$.}
      \label{fig_compare_1}
\end{figure}

$2)$ \emph{Effect of trajectory configuration}: Conduct the simulation under case $(i)$, where one of the variables $\|\bf t\|$ and $R_{max}$ is fixed, while observing the influence on the accuracy of another variable. The observation of Fig.~\ref{fig_tra} reveals that our method exhibits comparable performance to SDP in challenging scenarios described in~\cite{nguyen_relative_2023}, which are characterized by a large $\|\bf t\|$ and a small operating radius $R_{max}$.
\begin{figure}[thpb]
      \centering
      \includegraphics[width=\textwidth, keepaspectratio]{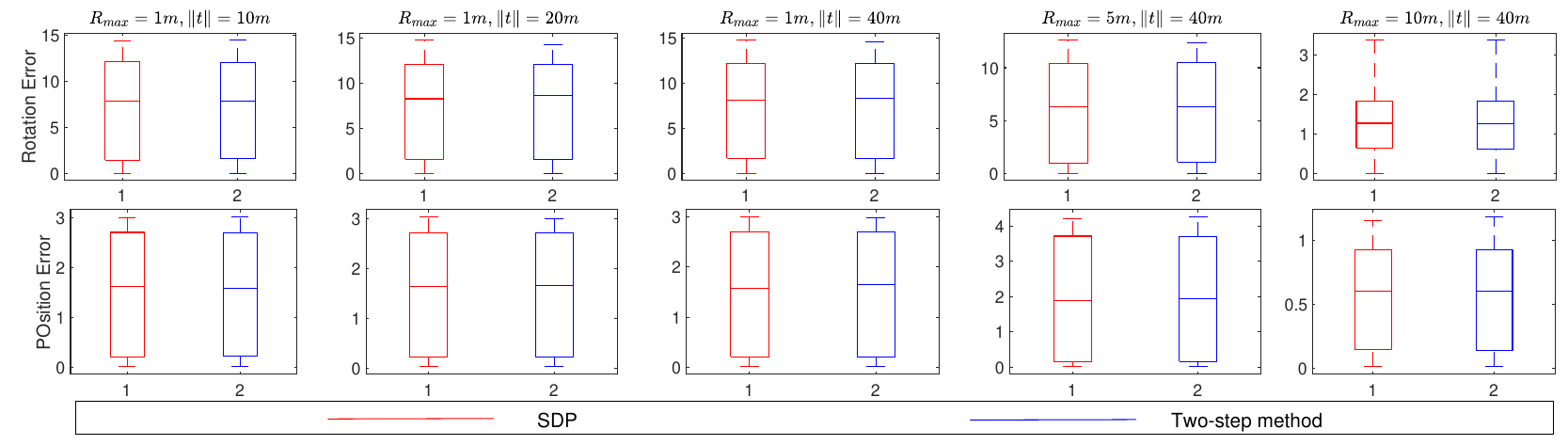}
      \caption{Estimation errors with varying $R_{max}$ and $\|\bf t\|$.}
      \label{fig_tra}
\end{figure}

$3)$ \emph{Effect of Speed levels}:
We have known $$\text{speed}~(v)=\cfrac{\text{length of the path}~(s_k-s_{k-1})}{\text{time}~(t_k-t_{k-1})}.$$
The speed of robot movement was set to $[v,2v,3v,4v,5v]$. The relationship between speed and error under case $(i)$ is illustrated in Fig.~\ref{fig_speed}, where it is observed that higher speeds are positively correlated with increased errors. This is because the displacement caused by the sampling interval results in large UWB errors at high velocities. And the susceptibility of our method to speed influence was observed, indicating its suitability for systems operating at lower speeds.
\begin{figure}[thpb]
      \centering
      \includegraphics[width=\textwidth, keepaspectratio]{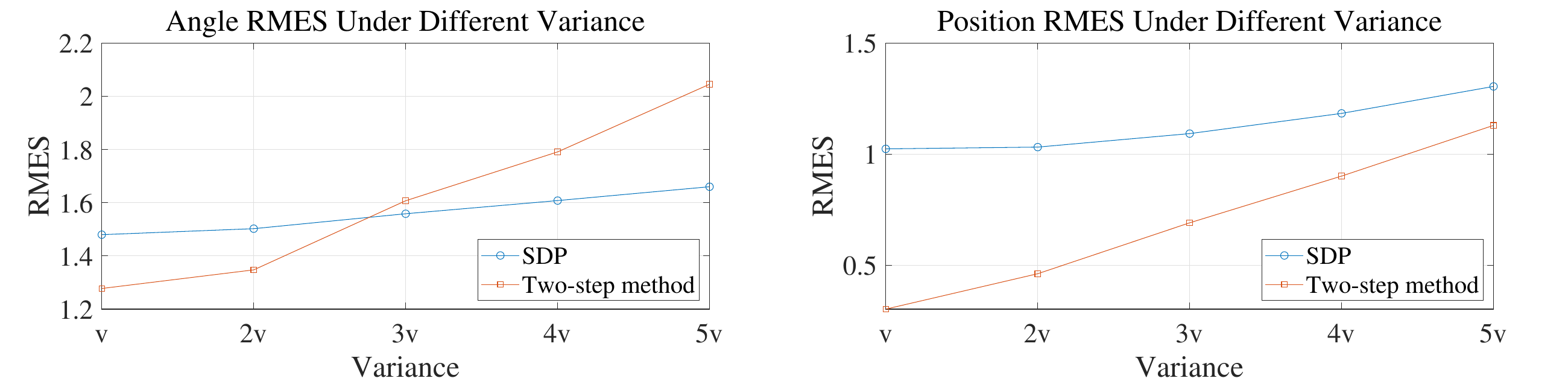}
      \caption{Estimation errors with varying speed.}
      \label{fig_speed}
\end{figure}

$4)$ \emph{Average computation time}:
The simulation conducted under case $(i)$, as shown in TABLE~\ref{table}, demonstrates the average operation time of the two methods. It is evident that the two-step method exhibits superior computational efficiency compared to SDP. This observation aligns with intuition as the closed form solution provided by the two-step method enables a rapid generation of a better initial value, thereby enhancing overall algorithm efficiency.
\begin{table}[h]  
\centering  
\caption{Average computation time of the SDP and two-step method.}  
\label{table}  
\begin{tabular}{c c}  
\hline  
method & Average computation time (s) \\  
\hline  
SDP & 0.3533 \\  
Two-step method &  0.0079\\ 
\hline  
\end{tabular}  
\end{table}  
%
%
%
%
%
\section{CONCLUSIONS AND FUTURE WORKS}\label{sec_conclusion}

We have presented a two-step estimation method for determining 4-DOF robot-robot RTE using odometers and distance measurements between robots. Our approach, based on maximum likelihood estimation, combines unconstrained least squares and projection techniques to obtain good initial estimates. Subsequently, we employ one step of Gauss-Newton iteration to achieve more accurate results efficiently. Through simulations and comparisons with the SDP method, we demonstrated that our proposed approach offers faster computation without compromising accuracy, making it suitable for solving RTE problems in constrained spaces. Furthermore, our method provides design principles for determining optimal installation positions and the minimum operating time based on the number of deployed UWB devices. This not only addresses real-time RTE but also presents a practical solution for scenarios requiring a specific number of UWB installations.

This paper primarily focuses on two-robot scenarios with a carefully designed path. We intend to extend the current algorithm to accommodate multiple agents and plan to explore techniques for reducing path constraints. Additionally, it is worth considering future work aimed at reducing hardware costs by enhancing the algorithm's capability to estimate pitch and roll angles.
\bibliographystyle{IEEEtran}
\bibliography{refs}
\end{document}